\documentclass[11pt]{article}

\date{}

\usepackage{times}

\usepackage[utf8]{inputenc} 
\usepackage[T1]{fontenc}    
\usepackage{hyperref}       
\usepackage{url}            
\usepackage{booktabs}       
\usepackage{amsfonts}       
\usepackage{nicefrac}       
\usepackage{microtype}      
\usepackage{xspace}

\usepackage{comment}
\usepackage{amsmath}
\usepackage{amsthm}
\usepackage{graphicx}
\usepackage{rotating}

\usepackage{amssymb}
\usepackage{autobreak}
\usepackage{mathtools}
\usepackage{xcolor}
\usepackage{bbm}
\usepackage{algorithm}

\usepackage[numbers, compress]{natbib}

\usepackage{amsmath,amsthm,amssymb}
\usepackage{amsfonts}
\usepackage{xcolor}         
\usepackage{graphicx}
\usepackage{caption}
\usepackage{mathtools}
\usepackage{enumitem}
\usepackage{bm}
\usepackage{algorithm}
\usepackage[noend]{algorithmic}

\usepackage{xspace}

\newcommand{\alg}{\textsc{\small{KCRL}}\xspace}

\newtheorem{assumption}{Assumption}

\newtheorem{theorem}{Theorem}

\newtheorem{proposition}{Proposition}

\newtheorem{definition}{Definition}
\newtheorem{remark}{Remark}

\newcommand{\dist}{\mathrm{dist}}

\usepackage{tikz}
\usetikzlibrary{shapes.geometric, arrows,positioning}

\tikzset{
  startstop/.style={
    rectangle, 
    rounded corners,
    minimum width=3cm, 
    minimum height=1cm,
    align=center, 
    draw=black
    },
  process/.style={
    rectangle, 
    minimum width=3cm, 
    minimum height=1cm, 
    align=center, 
    draw=black
    },
  decision/.style={
    rectangle, 
    minimum width=3cm, 
    minimum height=1cm, align=center, 
    draw=black
    },
  arrow/.style={draw,thick,->,>=stealth},
  dec/.style={
    ellipse, 
    align=center, 
    draw=black
    },
}

\title{ KCRL: Krasovskii-Constrained Reinforcement Learning with Guaranteed Stability in Nonlinear Dynamical Systems}

\author{%
  Sahin Lale\thanks{California Institute of Technology, \texttt{\{alale,adamw,anima\}@caltech.edu}} , Yuanyuan Shi\thanks{University of California San Diego , \texttt{ yyshi@eng.ucsd.edu}} , Guannan Qu\thanks{Carnegie Mellon University, \texttt{gqu@andrew.cmu.edu}} , Kamyar Azizzadenesheli\thanks{Purdue University, \texttt{kamyar@purdue.edu}} , \\ Adam Wierman\footnotemark[1] , Anima Anandkumar\footnotemark[1]
 
}

\begin{document}

\maketitle

\begin{abstract}
 Learning a dynamical system requires stabilizing the unknown dynamics to avoid state blow-ups. However, current reinforcement learning (RL) methods lack stabilization guarantees, which limits their applicability for the control of safety-critical systems. We propose a model-based RL framework with formal stability guarantees, Krasovskii Constrained RL (\alg), that adopts Krasovskii's family of Lyapunov functions as a stability constraint. The proposed method learns the system dynamics up to a confidence interval using feature representation, e.g. Random Fourier Features. It then solves a constrained policy optimization problem with a stability constraint based on Krasovskii's method using a primal-dual approach to recover a stabilizing policy. We show that \alg is guaranteed to learn a stabilizing policy in a finite number of interactions with the underlying unknown system. We also derive the sample complexity upper bound for stabilization of unknown nonlinear dynamical systems via the KCRL framework.
\end{abstract}

\section{Introduction}
\label{sec:intro}
Reinforcement Learning (RL) has achieved impressive success in the past decade, highlighted by success stories in game play~\cite{silver2016mastering}. As a result, RL has been recognized as a promising alternative for  traditional decision-making and control tasks in engineering systems, e.g.  robotics~\cite{levine2018learning}, energy systems~\cite{chen2021reinforcement}, and transportation~\cite{wu2017flow}. However, despite the promise, major hurdles remain before deployment in such systems is feasible.

One of the key challenges is that many real-world systems are safety-critical and have high standards for stability. 
Even though RL algorithms outperform classical control methods in complex and uncertain dynamical environments, 
they generally do not provide formal stability guarantees outside of linear dynamical systems \cite{lale2022reinforcement,qu2021stable}. In particular, most popular RL algorithms for control of nonlinear systems follow model-free gradient-based policies that focus on minimizing the control cost and do not explicitly consider stability \cite{lillicrap2015continuous, haarnoja2018soft}. This lack of stability guarantees currently prevents the deployment of RL algorithms in real-world problems, where the dynamics are usually nonlinear and instabilities are costly, e.g., frequency instability in power systems \cite{cui2020reinforcement}.

In contrast, control-theoretic approaches provide a rich set of tools for analyzing the stability of dynamical systems and synthesizing stable control policies. There is a large body of work that focuses on designing stable and robust controllers for linear systems and beyond \cite{zhou1996robust,khalil2002nonlinear}. Tools like Lyapunov's direct method, contraction analysis \cite{lohmiller1998contraction} and  passivity theory \cite{slotine1991applied} provide ways to verify stability and synthesize stabilizing controllers for nonlinear dynamical systems. The key challenge in these methods is to find the right Lyapunov functions to verify stability, which in practice heavily relies on trial and error. To this end, Krasovskii's approach~\cite{feijer2009krasovskii} gives a principled way to explicitly construct Lyapunov functions by considering a quadratic form of the derivative of the state variables.

\begin{figure*}[t]
    \centering
    \begin{tikzpicture}[
  node distance=0.4cm,
  every edge/.style={arrow}
  ]
\node (in1) [process] {RFF representation \\ of past \\ state-action pairs};
\node (out1) [process,right=of in1] {Model Learning \\ in RFF Basis: \\ $\hat{F}(\cdot)$};
\node (out3) [process,right=of out1] {Stability-Constrained \\ Policy Learning \\ $\frac{\partial \hat{F}}{\partial x} M \frac{\partial \hat{F}}{\partial x}-M \prec -\epsilon I$};
\node (out2) [process,right=of out3] {Deploy \\ stable policy: \\ $u_t = \pi_{\theta_i}(x_t)$};
\path
  (in1) edge (out1) 
  (out1) edge  (out3)
  (out3) edge node[below] {$\theta_i$ } (out2)
  (out2) edge[bend left=90, looseness=0.25] node[left]  {$\substack{\text{\small End of} \\ \text{\small Epoch}}$ } (in1);
  
\path[->,every loop/.style={min distance=6mm,looseness=1}] (out2)
         edge  [in=-250,out=70,loop] node[above] {$\substack{\text{\small During Epoch} }$} (); 
\end{tikzpicture}
    \caption{\alg~Framework}
    \label{fig:algo_framework}
\end{figure*}

The common denominator of these control-theoretic methods is that they require a highly accurate model of the system and assume that such a model is known a priori. Recently, there has been a growing interest in incorporating modeling errors in deriving various stability guarantees \cite{tsukamoto2020neural,singh2020learning,boffi2020learning}. 
However, these works neither provide guarantees in the \emph{online} learning/control setting, nor 
achieve the desired modeling errors and stability performance in \emph{finite time}, which are the key requirements for deploying learning-based methods in real-world systems. 



\textbf{Contributions.} 
Motivated by the contrast above, we integrate control theoretic tools into RL and provide a finite-time learning to stabilize guarantee in online control of unknown nonlinear dynamical systems. In particular, we propose a model-based RL framework, Krasovskii Constrained RL (\alg), that adapts Krasovskii's construction of quadratic Lyapunov functions as a stability constraint in the policy optimization problem. This novel stability constraint is a linear matrix inequality type constraint on the Jacobian of the closed-loop system, which guarantees that the Lyapunov stability conditions are met by design (Theorem \ref{thm:true_system_stability}). \alg learns the unknown model dynamics in epochs via kernel-based feature representations, i.e. Random Fourier Features (RFF)~\cite{rahimi2007random}, and solves the stability-constrained policy optimization problem via a primal-dual approach using the learned model in the constraint (Figure \ref{fig:algo_framework}).

To formally guarantee the learning of stabilizing control policies in finite-time, we show that (i) the Jacobian of the underlying closed-loop system can be learned up to a confidence interval using RFF (Proposition \ref{thm:RFF_jacobian_finite}), (ii) the solution of the policy optimization problem with the stability constraint that uses a learned model still stabilizes the underlying system (Theorem \ref{thm:estimated_system_stability}), and (iii) the primal-dual approach guarantees the satisfaction of this stability constraint (Theorem \ref{thm:primal_dual_guarantee}) after convergence. 
Combining these results, we provide the number of samples required to learn a stabilizing policy for the underlying system in terms of system properties and the number of RFF (Theorem \ref{thm:sample_complexity}). 




\section{Related Work}
This paper connects to a broad set of literature in RL and control.

\textit{Model-based RL in dynamical systems.} There has been a flurry of studies that consider model-based RL in dynamical systems due to their superior sample efficiency, task generalization and interpretable guarantees~\cite{recht2019tour}. The main focus has been on learning the system dynamics and providing performance guarantees in finite-time for both linear \cite{simchowitz2020naive} (and references within),
and nonlinear systems~\cite{kakade2020information,lale2021model}. While deriving these performance guarantees, the formal finite-time stability guarantees are also derived for linear systems~\cite{faradonbeh2018finite}. However, these guarantees have only been \textit{assumed to hold} with a stabilizing oracle for nonlinear systems~\cite{kakade2020information,lale2021model}. Therefore, our work completes an important part of the picture in finite-time stabilization of dynamical systems. 

\textit{Control Theory.} Our work is related to the Krasovskii's method~\cite{feijer2009krasovskii} and contraction analysis~\cite{lohmiller1998contraction}, which can be viewed as a specific way of constructing a Lyapunov function. More broadly, there are a large body of tools in control theory that synthesize controllers that are stable, e.g., passivity theory and feedback linearization~\cite{slotine1991applied}.
In the context of these works, our contribution can be viewed as bridging one of these classical tools with policy learning in RL, where the system model is unknown.

\textit{Lyapunov-based Policy Learning.} Lyapunov theory is a systematic framework to analyze the stability of a control system. The core idea is to identify a Lyapunov function of the state, with negative derivative along system trajectories. 
Using Lyapunov functions in RL was first introduced by~\cite{perkins2002lyapunov}, but the work did not discuss how to find a candidate Lyapunov function in general. 
A set of recent works including~\cite{chow2018lyapunov,richards2018lyapunov,chang2019neural,jin2020neural} have attempted to address this challenge by jointly learning the policy and the Lyapunov function.
Our goals and approaches are different. We use a principled framework to \textit{construct} (rather than learn) the Lyapunov function, and design the policy learning such that the stability conditions are met. In this sense, \cite{shi2021stability} is most related to our work. It considers a particular power system model and proposes model-free monotonic policies to stabilize the system. However, their model assumption is specific to power systems, whereas our work considers general nonlinear systems. 

\textit{Safe RL.} The goal of safe RL is to learn a policy that avoids entering unsafe regions, rather than requiring a policy to have attractive or bounded behavior. Various methods have been proposed to ensure safety in RL, 
such as using control barrier functions~\cite{taylor2020learning}, Hamilton-Jacobi safety analysis~\cite{fisac2019bridging}, and MPC as a safety filter~\cite{hewing2020learning}. We note that safety does not imply stability and vice visa, and our main focus in this paper is to learn an RL policy with the stability guarantee.

\section{Preliminaries}
\label{sec:preliminary}
In this section, we first describe the control problem and formally define the stability criteria. Then, we discuss the model dynamics and give a brief overview of RFF with their approximation guarantees.

\subsection{Control Problem}
Consider a discrete-time nonlinear system given as 
\begin{equation} \label{eq:system}
    x_{t+1} = f(x_t, u_t),
\end{equation}
where $x_t \in \mathbb{R}^n$ is the state of the system, $u_t \in \mathbb{R}^p$ is the control input at time-step $t$. We study the discrete optimal control setting for the system given in \eqref{eq:system}. Suppose there is a class of controllers $g_\theta(\cdot)$, parameterized by $\theta \in \Theta$. The goal is to design a controller $g_\theta(\cdot)$ that minimizes a control cost, 
\begin{subequations}\label{eq:policy_opt}
\begin{align}
    \min_{\theta} J(\theta) = &\sum\nolimits_{t=0}^\infty \gamma^t c(x_t,u_t), \\
    \text{ s.t. } x_{t+1} &= f(x_t, u_t), \label{eq:dyn}\\
    u_t &= g_\theta(x_t), \label{eq:feedback_ctrl}
\end{align}
\end{subequations}
where $c(x,u)$ is the per stage cost and $\gamma$ is the discounting factor. 

Note that there are many ways to solve or approximate the policy minimization problem \eqref{eq:policy_opt}. Generally speaking, the procedure is to run gradient methods on the policy parameter $\theta$ with step size $\eta$,
$ \theta \leftarrow \theta - \eta \nabla J(\theta). $
To approximate the gradient $\nabla J(\theta)$, one can use sampled trajectories such as REINFORCE or value function approximation such as actor-critic methods. As we are dealing with deterministic policies, one of the most popular choices is the Deep Deterministic Policy Gradient (DDPG)~\cite{lillicrap2015continuous}, where the policy gradient is approximated by
\[ \nabla J(\theta) \approx \frac{1}{N} \sum_{i\in B} \nabla_u \hat{Q}(x, u)|_{x = x_i, u=g_{\theta}(x_i)} \nabla_{\theta} g_{\theta}(x)|_{x_i}.\]
Here $\hat{Q}(x, u)$ is the value network (the critic) that can be learned via temporal difference learning, $g_{\theta}(x)$ is the actor network, and $\{x_i\}_{i\in B}$ are a batch of samples with batch size $|B| = N$ sampled from the replay buffer which stores historical state-action pairs. 

\subsection{Stability} 
In control systems, stability studies whether the state trajectory of the closed-loop system $x_{t+1} \!=\! f(x_t,g_\theta(x_t))$ asymptotically converges to the desired stationary point. 
However, in general RL formulations such as \eqref{eq:policy_opt}, stability is not an explicit requirement. It is mostly treated as an implicit regularization, since instability usually leads to high (or infinite) costs. 
The lack of an explicit stability requirement can lead to several issues. During the training phase, the policy may become unstable, causing the training process to terminate. Even after a policy is trained, there is no formal guarantee that the closed loop system is stable, which hinders the learned policy's deployment in real-world engineering systems where there is a very strict requirement of stability. 
In order to explicitly constrain stability in policy learning, we constrain the search space of policy in the set of stabilizing controllers from Lyapunov stability theory.

\begin{definition} (Asymptotically stable equilibrium)
A dynamical system $x_{t+1} = f(x_t,g_\theta(x_t))$ is asymptotically stable around $x^{(e)}$ if $f(x^{(e)},g_\theta(x^{(e)}))=x^{(e)}$, and further, there exists a region around $x^{(e)}$, $B_\delta(x^{(e)}) = \{x: \Vert x - x^{(e)}\Vert\leq \delta\} $ such that $\forall x_0\in B_\delta(x^{(e)})$, we have $\lim_{t\rightarrow \infty} \Vert x_t - x^{(e)}\Vert = 0 $. 
\end{definition}
More generally, the following definition considers a set of equilibrium points, where we have used notation $\dist(x, S ):= \inf_{y\in S} \Vert y - x\Vert $ to denote distance between point $x$ and set $S$. 
\begin{definition}(Asymptotically stable set)
A dynamical system $x_{t+1} = f(x_t,g_\theta(x_t))$ is asymptotically stable around set $S_e$ if $f(x^{(e)},g_\theta(x^{(e)}))=x^{(e)}, \forall x^{(e)}\in S_e$, and further, there exists $B_\delta(S_e) :=\{x:\dist(x,S_e)\leq \delta\}$ such that  $\forall x_0\in B_\delta(S_e)$, we have $\lim_{t\rightarrow \infty} \dist(x_t, S_e )  = 0 $.  
\end{definition}

A common approach to proving stability of a dynamical system is via Lyapunov's direct method, which involves defining a positive definite function that decreases along the system trajectory. For a more complete overview of Lyapunov stability theory, please refer to~\cite[Chapter 3]{khalil2002nonlinear}.

\subsection{System Dynamics} 
In this work, we study problem \eqref{eq:policy_opt} under unknown system dynamics. Note that for $\phi_t \!=\! [x_t^\top, u_t^\top]^\top$, one can write the system dynamics given in \eqref{eq:system} as \begin{equation} \label{eq:sys_vectorfield}
    x_{t+1} = F(\phi_t),
\end{equation}
for some nonlinear function $F$. For this given system, we assume that the unknown nonlinear function $F$ lives within Reproducing Kernel Hilbert Spaces (RKHS) of infinitely smooth functions defined by a known positive definite continuous kernel $\kappa(\cdot, \cdot)$, e.g. Gaussian kernel. In particular, we assume that each mapping of $\phi_t$ to the elements of state vector $x_{t+1}$, \textit{i.e.} $(x_{t+1})_i = F_i(\phi_t)$ for $i=1,\ldots, n$, lives in this known RKHS.

Note that Gaussian kernels are universal kernels such that they can approximate an arbitrary continuous target function uniformly on any compact subset of the input space using possibly infinite kernel evaluations~\cite{micchelli2006universal}. Therefore, the class of nonlinear dynamics considered in this work, which can be infinite sum of kernel evaluations, constitutes a vast variety of nonlinear systems. In particular, they are more general than the kernelized nonlinear systems, where the system dynamics are assumed to be represented by a finite sum of kernel evaluations~\cite{kakade2020information,mania2020active}.

\begin{remark}
We consider fully observable (state-feedback) nonlinear systems in this paper. However, the results of this work can be extended to partially observable nonlinear dynamical systems. More specifically, the results can be rewritten for an order-$h$ nonlinear autoregressive exogenous system as depicted in \cite{lale2021model}, \textit{i.e.}, considering the last $h$ input-output pairs as the state $x_t$. 
\end{remark}


\subsection{RFF and Approximation Guarantees}

Kernel methods are foundational tools used in modeling complicated functional relationships in many machine learning problems. The so-called ``kernel trick'' is the main component of these methods. The kernel trick states that for some positive definite kernel $\kappa(\cdot, \cdot)$, the kernel evaluation of data points $x_1$ and $x_2$ are equivalent to inner product between possibly infinite dimensional feature representations $\psi(\cdot)$ of $x_1$ and $x_2$ in a Hilbert space $\mathcal{H}$: $\kappa(x_1,x_2) \!=\! \langle \psi(x_1), \psi(x_2) \rangle_{\mathcal{H}} $.

Assume that we have an underlying functional $h$ such that $y_i \!=\! h(x_i)$ and some collected data pairs $\mathcal{D}\!=\! (x_i,y_i)_{i=0}^K$ for $x \!\in\! \mathbb{R}^{d_x}$ and $y \!\in\! \mathbb{R}$. The kernel methods construct nonlinear models as $\hat{h}(\cdot) \!=\! \sum_{i=1}^K \alpha_i \kappa(x_i, \cdot)$, such that $\alpha_i$ are chosen to best represent $\mathcal{D}$ for some kernel $\kappa(\cdot,\cdot)$. However, for large number of data points, solving for $\alpha$ is computationally expensive. For this very reason, \cite{rahimi2007random} proposes to approximate $\psi(\cdot)$ with finite $D$-dimensional features $z(\cdot)$: 
\begin{equation} \label{eq:RFF_approx}
    \hat{h}(\cdot) = \sum\nolimits_{i=1}^K \alpha_i\langle \psi(x_i), \psi(\cdot) \rangle_{\mathcal{H}} \approx  \sum\nolimits_{i=1}^K \alpha_i z(x_i)^\top z(\cdot).
\end{equation}
This feature representation $z(x)$ of $x$ is termed as Random Fourier Features (RFF) and generated as 
\begin{equation} \label{eq:RFFgeneration}
z(x) \coloneqq \sqrt{\frac{2}{D}}\left[\cos \left(\omega_{1}^{\top} x+b_{1}\right),\ldots,\cos \left(\omega_{D}^{\top} x+b_{D}\right) \right]^\top
\end{equation}
where $\omega_{i}$ are drawn iid from the normalized Fourier transform of the kernel $\kappa$, which corresponds to a proper distribution $p(\omega)$, and $b_i$ are drawn iid from uniform distribution on $[0,2 \pi]$. Furthermore, \cite{rahimi2007random} shows that this method provides an unbiased estimate of $\kappa$ and the approximation error decays exponentially in $D$ (Claim 1 of \cite{rahimi2007random}), which motivates the use of RFFs in function approximation in practice~\cite{kakade2020information}.

Notice that so far $h(\cdot)$ considered in this section is a scalar-valued function. Recently, \cite{lale2021model} derived approximation theoretic guarantees for 
using RFF to approximate vector-valued nonlinear functions that live in a known RKHS within a bounded region, such as $F(\cdot)$ in \eqref{eq:sys_vectorfield}. In particular, \cite{lale2021model} shows that for large enough number of RFF ($D$), the best $D$-dimensional RFF approximation of $F$, $\bar{F}(\cdot) = W_*^\top z(\cdot) $, attains the approximation guarantee of
\begin{equation} \label{eq:bestRFFapprox}
\sup_{\|\phi\| \leq \Gamma_\phi} \|\bar{F}(\phi) - F(\phi) \| \leq \tilde{\mathcal{O}}(1/\sqrt{D})    
\end{equation}
with high probability, where $\Gamma$ describes the bounded region. Here $\tilde{\mathcal{O}}(\cdot)$ denotes the order up to logarithmic factors of $D$ and hides the dependencies on $n$ and $\Gamma_\phi$. 

This result is key to our analysis as we use it to derive the finite-time learning and stabilization guarantees of \alg in Section \ref{sec:theory_results}. In particular, we use this result to derive a novel finite-time approximation guarantee for the Jacobian of the underlying system \eqref{eq:sys_vectorfield} using RFF, which is then used to guarantee the design of a finite-sample stabilizing controller via \alg.



\section{KCRL Framework}
\label{sec:framework}
In this section, we present \alg. 
The outline of the algorithm is given in Algorithm \ref{algo:FLS}. \alg works in epochs of length $\tau$, where the controller is during the epoch is fixed. Each epoch consists of two parts:
\begin{enumerate}
    \item \emph{Model Learning:} We generate trajectories from the underlying system using the current controller, and use these trajectories to update the model estimate.
    \item \emph{Stable Policy Design:} We use a primal dual policy gradient approach to solve a stability constrained policy optimization problem and design the new controller for the next epoch.
\end{enumerate}

\begin{algorithm}[t] 
\caption{\alg }
  \begin{algorithmic}[1]
 \STATE \textbf{Input:} $\tau$, $g_{\theta_{0}}$, $D$, $\lambda$, $\epsilon_i$, $M$, $\mu$, $\eta_1$, $\eta_2$, $\epsilon_{pd}$  \\

\FOR{$i = 0, \ldots $}
    \FOR{$t = i \tau, \ldots, (i+1)\tau$}
    
        \STATE Execute $u_t = g_{\theta_{i}} (x_t)$ 
        \STATE Form $\phi_t = [x_t^\top, u_t^\top]^\top$ and store $z(\phi_{t})$ 
    \ENDFOR
    \STATE Solve \eqref{eq:least-squares} for $\hat{W}_i$ \& Form $\hat{F}_i(\cdot)=\hat{W}_i^\top z(\cdot) $ \hfill $\rhd$  \textbf{Model Learning} \\
    \STATE Solve \eqref{eq:policy_opt_stab} for $\theta_{i+1}$ using $\hat{F}_i(\cdot)$ via \eqref{eq:primal_dual}  \\
    \STATE Construct $g_{\theta_{i+1}}$ \hfill $\rhd$  \textbf{Stable Policy Design}
    
\ENDFOR
  \end{algorithmic}
 \label{algo:FLS} 
\end{algorithm}

\subsection{Model Learning}
Each epoch $i$ of \alg starts with a data collection from the underlying system for $\tau$ time-steps using the current controller, $g_{\theta_{i}}(\cdot)$. In each time step, \alg takes the action $u_t = g_{\theta_{i}} (x_t)$, it 
computes and stores the $D$-dimensional RFF representation of the current state-action pair $\phi_t = [x_t^\top, u_t^\top]^\top$ as $z(\phi_t)$ using \eqref{eq:RFFgeneration}. Note that $D$ and $\tau$ are user-defined parameters and the policy in the initial epoch is $g_{\theta_{0}}(\cdot)$.

Using the best $D$-dimensional RFF approximation of $F$ defined in \eqref{eq:bestRFFapprox}, we can approximate \eqref{eq:sys_vectorfield} as $x_{t+1} \approx W_*^\top z(\phi_t)$, for some unknown $W_* \in \mathbb{R}^{D \times n}$. For model learning, \alg considers this approximate model and tries to recover the best estimate for $W_*$ using all the data gathered. In particular, after the data collection of epoch $i$, \alg solves the following regularized least squares problem:
\begin{equation} \label{eq:least-squares}
    \min_{W} \lambda \| W \|_F^2 + \sum\nolimits_{s=0}^{t=(i+1)\tau} \|x_{s+1} - W^\top z(\phi_s)\|^2_2,
\end{equation}
for some $\lambda>0$ to obtain an estimate of $W_*$. Note that $\hat{W}_i = (Z_t Z_t^\top + \lambda I)^{-1} Z_t X_t^\top$ gives the closed-form solution of \eqref{eq:least-squares} for $X_t \!=\! [x_{t+1}, \ldots, x_{1}] \!\in\! \mathbb{R}^{n \times {(t+1)}}$, $Z_t \!=\! [z(\phi_t), \ldots, z(\phi_{0}) ] \!\in\! \mathbb{R}^{D \times {(t+1)}}$. Thus, at epoch $i$, the learned model by \alg is given by $\hat{F}_i(\cdot) = \hat{W}_i^\top z(\cdot)$. Note that instead of solving \eqref{eq:least-squares} from scratch in every epoch, \alg can also use online updates for constructing $\hat{W}_i$.

\subsection{Stable Policy Design}
Once \alg has a model estimate after the data collection, it aims to recover a stabilizing policy for the underlying system. To this end, it solves the following stability-constrained policy optimization problem at epoch $i$ using the estimated model $\hat{F}_i(\cdot)$:
\begin{subequations}\label{eq:policy_opt_stab}
    \begin{align}
    \min_{\theta} J(\theta) = &\sum\nolimits_{t=0}^\infty \gamma^t c(x_t,u_t), \\
    \text{ s.t. } x_{t+1} &= F(\phi_t), \label{eq:sysdyn}\\
    u_t &= g_\theta(x_t), \\
     \hat{G}_i(x,\theta)^\top M \hat{G}_i(x,\theta) - M  &\prec -\epsilon_i I, \forall x\in\mathbb{R}^n ,  \label{eq:stability_LMI_Kra}
\end{align}
\end{subequations}
where $\hat{G}_i(x,\theta) = \widehat{ \frac{\partial F(\phi)}{\partial x}} + \widehat{\frac{\partial F(\phi)}{\partial u}} \frac{\partial u}{\partial x}$ for the Jacobian estimates $\widehat{ \frac{\partial F(\phi)}{\partial x}}$ and $\widehat{\frac{\partial F(\phi)}{\partial u}}$ that are computed via finite difference method on the estimated model $\hat{F}_i(\cdot)$, $M  \succ  0$ is an appropriately chosen positive definite matrix and $\epsilon_i > 0$ is chosen with respect to the model learning error at epoch $i$. 

Compared to \eqref{eq:policy_opt}, the formulation \eqref{eq:policy_opt_stab} incorporates an additional constraint \eqref{eq:stability_LMI_Kra}. This constraint adapts Krasovskii's method for Lyapunov function construction and helps \alg to enforce stability of the learned policy. In Section \ref{sec:theory_results}, we provide a formal treatment of this constraint. In particular, we will first show that solving \eqref{eq:policy_opt_stab} using the true system $F(\cdot)$ in \eqref{eq:stability_LMI_Kra} gives a stabilizing policy. We further show that solving \eqref{eq:policy_opt_stab} using a well-refined estimate $\hat{F}_i$ and an appropriate choice of $\epsilon_i$ guarantees the recovery of stabilizing policy for the underlying system $F(\cdot)$. To solve the constrained optimization problem \eqref{eq:policy_opt_stab}, \alg uses a primal-dual technique. 

\textbf{Primal-Dual Approach.} To describe the approach, we use the following short-hand notation, $K(x,\theta)\!=\! \hat{G}(x,\theta)^\top M \hat{G}(x,\theta) \!-\! M \!+\! \epsilon_i I.$
With this notation, \eqref{eq:policy_opt_stab} can be reformulated as
\[ \min_{\theta} J(\theta) \text{ s.t. } \sup_x \lambda_{\max}( K(x,\theta)) < 0,\]
where $\lambda_{\max}(\cdot)$ is the largest eigenvalue. The Lagrangian for the problem is given as,
\[L(\theta,\mu) = J(\theta) + \mu \sup_x \lambda_{\max}( K(x,\theta)). \]
The primal-dual algorithm then proceeds as follows \cite{nedic2009subgradient},
\begin{align*}
    \theta&\leftarrow \theta - \eta_1 \Big[\nabla J(\theta)+ \mu \nabla_\theta  \sup_{x}\lambda_{\max}(  K(x,\theta)) \Big], \\
    
\mu &\leftarrow \max(0, \mu +   \eta_2 \sup_{x}\lambda_{\max}(  K(x,\theta) )).
\end{align*}
Since it is not possible to evaluate the $\sup_x$, we replace the sup over the state space with a sup over a batch of  representative points in the state space $\{x_i\}_{i\in \mathcal{B}}$. For the term $\nabla J(\theta)$, \alg uses standard policy gradient estimators, e.g. DDPG~\cite{lillicrap2015continuous}, to evaluate the policy gradient and denote the estimated gradient as $\widehat{\nabla J(\theta)}$. Thus, the primal-dual algorithm of \alg is given as,
\begin{align}
    \theta&\leftarrow \theta - \eta_1 \Big[ 
    \widehat{\nabla J(\theta)}+ \mu \nabla_\theta  \sup_{i\in \mathcal{B}}\lambda_{\max}(  K(x_i,\theta)) \Big], \nonumber \\
\mu&\leftarrow \max(0, \mu +   \eta_2 \sup_{i\in \mathcal{B}}\lambda_{\max}(  K(x_i,\theta) + \epsilon_{\mathrm{pd}}  I )), \label{eq:primal_dual}
\end{align}
where $\eta_1, \eta_2 >0$ are the step sizes and $\mathcal{B}$ is a batch of representative points in the state space, and $\epsilon_{\mathrm{pd}}>0$ is another constant that is chosen to tolerate the possible representation incapability of $\mathcal{B}$. In Theorem \ref{eq:primal_dual}, we give the characterization of $\epsilon_{\mathrm{pd}}$ to provide the theoretical guarantees of \alg.

\section{Main Result: Stability Guarantee of KCRL}
\label{sec:theory_results}
After presenting the algorithmic details of \alg in the previous section, we now provide the theoretical guarantees of \alg.
First, we have the following learning and stabilizability assumptions for the underlying system $F(\cdot)$.

\begin{assumption} [Exploratory and Bounded Initial Policy] \label{asm:learning}
The initial controller $g_{\theta_0}$ provides persistently exciting (PE) and bounded inputs that can be used for exploration and excite the system uniformly. In other words, the smallest eigenvalue of the design (sample covariance) matrix $Z_tZ_t^\top$ scales linearly over time. Moreover, for the closed-loop system generated by using the controller $u_t = g_{\theta_0}(x_t)$, i.e. $F(\phi_t)$ with $\phi_t = [x_t^\top, g_{\theta_0}(x)^\top]^\top $, we have $\|\phi_t\|\leq \Gamma_{\phi}$, for some finite $\Gamma_{\phi}$. 
\end{assumption}

\begin{assumption}[Krasovskii's Lyapunov Function for $F$]\label{asm:stable}
Let $G(x,\theta)$ denote the true Jacobian of the closed-loop system with respect to state $x$, i.e., $G(x,\theta) = \frac{\partial F(\phi)}{\partial x} + \frac{\partial F(\phi)}{\partial u} \frac{\partial u}{\partial x}$. We assume that there exists an $(M, \theta)$ pair such that 
$G(x,\theta)^\top M G(x,\theta) - M  \prec -\bar{\epsilon} I$, for some $\bar{\epsilon}>0$.
\end{assumption}

\begin{assumption}[Regularity Conditions]\label{asm:regularity} (i) Both $F$ and $g_\theta$ are continuously differentiable. 
(ii) $F$ is $L_F$-Lipschitz, or in other words, we have Jacobian of $F$, $\Vert J_F\Vert \leq L_F$. (iii) $\|\nabla^2 F_i\| \leq \mathbf{F_H}$ for all $i$, where $F_i$ denotes the mapping from $\phi_t$ to $i$th element of state vector $x_{t+1}$ for all $t$, \textit{i.e.}, $(x_{t+1})_i = F_i(\phi_t)$ for $i=1,\ldots, n$.  (iv) Any policy in the considered policy class is always $L_u$-Lipschitz, that is, $\Vert \frac{\partial g_\theta(x) }{\partial x}\Vert \leq L_u$, $\forall \theta$. 
\end{assumption}

The first assumption in Assumption \ref{asm:learning} is fairly standard and it guarantees the consistent and reliable estimation of the underlying system, whereas the second assumption can be interpreted as bounded-input to bounded-state condition on the initial controller. Note that in many nonlinear control systems this condition is already satisfied due to physical laws, i.e. the state of the system cannot go unboundedly. Assumption \ref{asm:stable} guarantees that the underlying system can be made stable with respect to Krasovskii's Lyapunov function with a stability margin, $\bar{\epsilon}$. While this assumption applies in many systems, including network congestion control~\cite{feijer2010stability} and power system control~\cite{shi2021stability}, it is an important future direction to relax this assumption and consider how to incorporate other Lyapunov function construction techniques.  
Finally, Assumption \ref{asm:regularity} provides standard bounds on the system properties in order to provide the analysis. 

With these assumptions in hand, we are ready to present our first result that the solution of the novel stability-constrained policy optimization problem \eqref{eq:policy_opt_stab} with the perfect knowledge of the underlying system, particularly the true Jacobian $J_F$, is a stabilizing policy by design. 

\begin{theorem}[Stability of the True Discrete-time System] \label{thm:true_system_stability}
 Suppose Assumptions \ref{asm:learning}-\ref{asm:regularity} hold. Consider solving \eqref{eq:policy_opt_stab} with the knowledge of true model $F(\cdot)$, such that \eqref{eq:stability_LMI_Kra} is evaluated using the true Jacobian $J_F$. If \eqref{eq:stability_LMI_Kra} holds for $\epsilon_i = 0$, we have the trajectory of $x_{t+1} = f(x_t,g_\theta(x_t))$ is asymptotically stable around set $S_e = \{x:f(x,g(x)) = x\}$. 
\end{theorem}

The proof is in Appendix of the extended version online. The key idea that underpins this result is to construct Krasovskii's Lyapunov function for the system as $V(x) \!=\! (x\!-\!f(x,g_{\theta}(x)))^\top M (x\!-\!f(x,g_{\theta}(x)))$ and using Kowalewki's mean value theorem~\cite{kowalewski1895mittelwertsatz} to show that difference equation along the system trajectory is negative definite, which leads to our stability constraint in \eqref{eq:stability_LMI_Kra}. Note that Theorem \ref{thm:true_system_stability} uses the exact Jacobians rather than estimates obtained via finite difference method. The following extends this result to tolerate modeling errors, in particular errors in the Jacobian estimates. 

\begin{theorem}[Stability under Modeling Error] \label{thm:estimated_system_stability}
 Suppose Assumptions \ref{asm:learning}-\ref{asm:regularity} hold and the Jacobian estimation errors satisfy 
 $\sup_{x} \max(\Vert \widehat{ \frac{\partial F_i(\phi)}{\partial x}}  - \frac{\partial F_i(\phi)}{\partial x} \Vert, \Vert \widehat{\frac{\partial F_i(\phi)}{\partial u}}  - \frac{\partial F_i(\phi)}{\partial u} \Vert) \leq \varepsilon_J < 1$, for all $i=1,\ldots,n$.
 Then, when \eqref{eq:stability_LMI_Kra} holds with $\epsilon_i = 2\bar{G} \Vert M\Vert (1 + L_u) \varepsilon_J + \Vert M\Vert  (1 + L_u)^2 \varepsilon_J^2  $, where $\bar{G} = (1+L_u) (L_F + \varepsilon_J)$, we have the trajectory of $x_{t+1} = f(x_t,g_\theta(x_t))$ is asymptotically stable around set $S_e = \{x:f(x,g(x)) = x\}$. 
\end{theorem}
\begin{proof}
By Theorem~\ref{thm:true_system_stability}, we only need to show the following: (here we drop the dependence on $\theta$ as it is fixed throughout the proof)
\begin{align}
    G(x)^\top M G(x) - M \prec 0, \forall x. \label{eq:stab_error:condition}
\end{align}
Let $\Delta G_i \coloneqq \hat{G}_i(x) \!-\! G_i(x)$. We first bound $\Vert \Delta G_i \Vert$:
\begin{align}
    \Vert \Delta G_i\Vert &= \left\Vert \widehat{ \frac{\partial F_i(\phi)}{\partial x}}  \!-\! \frac{\partial F_i(\phi)}{\partial x} \!+\! \Big(\widehat{\frac{\partial F_i(\phi)}{\partial u}}  \!-\! \frac{\partial F_i(\phi)}{\partial u}\Big) \frac{\partial u}{\partial x}\right\Vert \nonumber\\ &\leq (1 + L_u) \varepsilon_J. \label{eq:upperboundondiff} 
\end{align}
Next, note that the stability constraint~\eqref{eq:stability_LMI_Kra} indicates
\[ \hat G_i(x)^\top M \hat G_i(x) - M \prec -\epsilon_i I, \forall x.\]
Using this, we get 
\begin{align*}
    G_i(x)^\top M G_i(x) &= \hat G_i(x)^\top M \hat G_i(x) + \Delta G_i^\top M \hat{G}(x) + \hat{G}_i(x)^\top M \Delta G_i + \Delta G_i^\top M \Delta G_i\\
    &\preceq M - \epsilon_i I+ 2 \bar{G} \Vert M\Vert \Vert \Delta G_i\Vert I + \Vert M\Vert \Vert \Delta G_i\Vert^2 I\\
    &\prec M,
\end{align*}
where in the final step, we use \eqref{eq:upperboundondiff} and the choice of $\epsilon_i$. This verifies \eqref{eq:stab_error:condition} and gives the advertised result. 
\end{proof}

Next, we turn our attention to the learning guarantees of \alg. In particular, we need to guarantee that the model estimation errors are small enough at the end of first epoch such that the controller obtained via solving \eqref{eq:policy_opt_stab} would stabilize the system. Recall that RFF representation translates the nonlinear system in \eqref{eq:sys_vectorfield} into a linear system of the form $W_*^\top z(\phi_t)$, where $z(\phi)$ is the $D$-dimensional RFF representation of $\phi$. Thus, using standard least-squares estimation error results for the solution of \eqref{eq:least-squares}, in particular Theorem 1 and 2 of \cite{lale2021model}, and under Assumption \ref{asm:learning}, for large enough $D$, we get 
\begin{equation} \label{eq:rff_finite}
    \sup_{\| \phi\|\leq \Gamma_\phi} \|F(\phi) - \hat{F}_1(\phi)\| = \tilde{\mathcal{O}}(1/\sqrt{D}+\sqrt{D/\tau}),
\end{equation}
after $\tau$ time-steps, \textit{i.e.}, at the end of first epoch of \alg.

Using \eqref{eq:rff_finite}, we derive the following novel finite sample approximation error guarantee on the Jacobian of the underlying function $F(\cdot)$ via the finite difference method for the computation of partial derivatives. 
\begin{proposition}[Approximation Error of Jacobian using RFF] \label{thm:RFF_jacobian_finite}
Let $J_F$ denote the Jacobian of the underlying system $F$ given in \eqref{eq:sys_vectorfield}. Consider the finite difference approximation of $J_F$ using $\hat{F}_1(\cdot) = \hat{W}_1^\top z(\cdot)$, i.e., the estimated system dynamics at the end of the first epoch, such that 
\begin{equation}
\hat{J}_F^{(i,j)}(\phi) = \frac{\hat{F}_{1,i}(\phi+\varepsilon~ \mathbf{e_j}) - \hat{F}_{1,i}(\phi-\varepsilon~ \mathbf{e_j})}{2\varepsilon}, \label{eq:jac_F_estimate}
\end{equation}
where $\varepsilon>0$, $\hat{F}_{1,i}(\cdot)$ is the mapping from input to the $i$th index of the output of $\hat{F}_1$ and $\mathbf{e_j}$ is the $j$th standard basis. Under Assumptions \ref{asm:learning} \& \ref{asm:regularity}, for the optimal choice of $\varepsilon=\mathcal{\tilde{O}}\left((1/\sqrt{D}+\sqrt{D/\tau})^{1/3}\right)$, we have that $\sup_{\|\phi\| \leq B} \|\hat{J}_F(\phi) - J_F(\phi) \|_F = \mathcal{\tilde{O}}\left(\varepsilon^{2}\right).$
\end{proposition}

\begin{proof}
Consider the following
\begin{align*}
    F_i(\phi+\varepsilon~ \mathbf{e_j}) &= F_i(\phi) + \varepsilon \frac{\partial F_i }{\partial \phi_j} + \varepsilon^2 \frac{\partial^2 F_i }{\partial \phi_j^2} + \mathcal{O}(\varepsilon^3) \\ 
    F_i(\phi-\varepsilon~ \mathbf{e_j}) &= F_i(x) - \varepsilon \frac{\partial F_i }{\partial \phi_j} + \varepsilon^2 \frac{\partial^2 F_i }{\partial \phi_j^2} - \mathcal{O}(\varepsilon^3).
\end{align*}
From Assumption \ref{asm:regularity}, we have
\begin{equation}
    \frac{F_i(\phi+\varepsilon~ \mathbf{e_j}) -  F_i(\phi-\varepsilon~ \mathbf{e_j})}{2\varepsilon} = \frac{\partial F_i }{\partial \phi_j}  + \mathbf{F_H} \mathcal{O}(\varepsilon^2). \label{eq:true_difference_F}
\end{equation}
Now consider \eqref{eq:jac_F_estimate}. Let $\delta_\tau = \tilde{\mathcal{O}}(1/\sqrt{D}+\sqrt{D/\tau})$. From \eqref{eq:rff_finite}, we have $\hat{F}_{1,i}(\phi+\varepsilon~ \mathbf{e_j})   = F_i(\phi+\varepsilon~ \mathbf{e_j}) + \epsilon_1$ and $\hat{F}_{1,i}(\phi-\varepsilon~ \mathbf{e_j})   = F_i(\phi-\varepsilon~ \mathbf{e_j}) + \epsilon_2$ for $0\leq \epsilon_1, \epsilon_2 \leq \delta_\tau$. Combining this with \eqref{eq:true_difference_F}, we obtain 
\begin{align*}
    \hat{J}_F^{(i,j)}(\phi) -  \frac{\partial F_i }{\partial \phi_j} &= \frac{\hat{F}_{1,i}(\phi+\varepsilon~ \mathbf{e_j}) - \hat{F}_{1,i}(\phi-\varepsilon~ \mathbf{e_j}) }{2\varepsilon} -  \frac{\partial F_i }{\partial \phi_j} \\ &= \frac{F_i(\phi+\varepsilon~ \mathbf{e_j}) + \epsilon_1 - F_i(\phi-\varepsilon~ \mathbf{e_j}) - \epsilon_2}{2\varepsilon} -  \frac{\partial F_i }{\partial \phi_j}\\
    &= \mathbf{F_H} \mathcal{O}(\varepsilon^2) + \frac{\epsilon_1 - \epsilon_2}{2\varepsilon}.
\end{align*}
This gives us that $|\hat{J}_F^{(i,j)}(\phi) -  \frac{\partial F_i }{\partial \phi_j}| \leq \mathbf{F_H}\mathcal{O}(\varepsilon^2) + \delta_\tau /\varepsilon $ for all $i,j$. Combining these yields 
\begin{equation*}
    \| \hat{J}_F(\phi) - J_F(\phi) \|_F \leq n\left(c\mathbf{F_H}\varepsilon^2 +  \frac{\delta_\tau}{\varepsilon} \right),
\end{equation*}
for some constant $c$. Note that the optimal choice of $\varepsilon$ is $\varepsilon=\mathcal{\tilde{O}}\left((1/\sqrt{D}+\sqrt{D/\tau})^{1/3}\right)$, which gives 
\[
\| \hat{J}_F(\phi) - J_F(\phi) \|_F = \mathcal{\tilde{O}}\left((1/\sqrt{D}+\sqrt{D/\tau})^{2/3}\right).
\]
\end{proof}
This result shows that the Jacobian of a vector-valued function in a known RKHS is well-approximated using the RFF representation of the function with finite samples, which can be of independent interest. Before stating our sample complexity result for stabilizing the system dynamics via \alg, we first show that the convergence of primal-dual method in \eqref{eq:primal_dual} guarantees the satisfaction of stability condition \eqref{eq:stability_LMI_Kra} for all points with proper choice of $\epsilon_{pd}$ and $\epsilon_i$.

\begin{theorem}[Primal-Dual Convergence Guarantees Stability] \label{thm:primal_dual_guarantee}
Suppose the primal-dual procedure converges, then the stability condition will be met for all samples in the batch of representative points $\mathcal{B}$ given in \eqref{eq:primal_dual}. Further, let $\mathcal{X}$ denote the bounded compact state space of the system for $\|\phi\|\leq\Gamma_\phi$ and suppose the batch $\mathcal{B} = \{x_i\}_{i=1}^{N}$ contains a finite set of points in $\mathcal{X}$ such that, $\forall x \in \mathcal{X}\,, \exists x_i \in \mathcal{B}\,, ||x- x_i|| < h$, for some $h>0$. Under the conditions of Theorem \ref{thm:estimated_system_stability}, for $||\frac{\partial \hat{G} (x, \theta)}{\partial x}|| \leq M_{G}$, if \alg sets $\epsilon_{\mathrm{pd}} = 2\bar{G}||M|| M_G h  $ in Eq~\eqref{eq:primal_dual}, then the stability condition is met on the entire state space $\mathcal{X}$ by choosing $\epsilon_i= 2\bar{G} \Vert M\Vert (1 + L_u) \varepsilon_J + \Vert M\Vert  (1 + L_u)^2 \varepsilon_J^2 $ in the constraint \eqref{eq:stability_LMI_Kra}.
\end{theorem}

\begin{remark}
Recall from Section \ref{sec:framework} that the batch $\mathcal{B}$ constitutes arbitrary points in state-space to estimate the supremum of the stability constraint and \textit{does not} correspond to data collected from the system. In particular, \alg only requires the evaluation in \eqref{eq:primal_dual} using the estimated system dynamics in these particularly chosen representative points in the state-space $\mathcal{X}$. Note that $h$ corresponds to the fill distance for the batch $\mathcal{B}$. This condition can be met by using $N = (\Gamma_\phi / h + 1 )^d$ samples in the batch $\mathcal{B}$. Furthermore, in the expense of computational burden, $N$ can be even picked larger which would in return reduce $h$ and consequently shrink $\epsilon_{pd}$ arbitrarily.
\end{remark}

\begin{proof}
Recall the definition of $K(x,\theta)\!=\! \hat{G}(x,\theta)^\top M \hat{G}(x,\theta) \!-\! M \!+\! \epsilon_i I$. The dual variable update in \eqref{eq:primal_dual} follows $\mu \leftarrow \max(0, \mu + \eta_2 \sup_{i\in \mathcal{B}}\lambda_{\max}(  K(x_i,\theta) + \epsilon_{\mathrm{pd}}  I )),$ and convergence of $(\mu,\theta)$ to $(\mu^*,\theta^*)$ implies $\forall i\in \mathcal{B}, \lambda_{\max}(  K(x_i,\theta^*) )) \leq 0$, that is the Krasovskii’s stability condition holds for all the samples in the batch. 

Since each batch $\mathcal{B}$ is drawn from the state space $\mathcal{X}$, as training time goes to infinity, the stability condition $\lambda_{\max}(  K(x_i,\theta^{*}) )) \leq 0$ also holds for all $x_i \in \mathcal{B}$. By the fill distance condition, \textit{i.e.}, $\forall x \in \mathcal{X}\,, \exists x_i \in \mathcal{B}\,, ||x- x_i|| < h$, (here we drop the dependence on $\theta^{*}$ as it is fixed throughout the proof) 
\begin{align}
    &\min_{x_i \in \mathcal{B}} ||K(x) \!-\! K(x_i)|| \!=\! \min_{x_i \in \mathcal{B}} ||\hat{G}(x)^\top M \hat{G}(x) - \hat{G}(x_i)^\top M \hat{G}(x_i)|| \nonumber\\
    & \! \leq \! \min_{x_i \in \mathcal{B}_{h}} \!\! ||(\hat{G}(x) \!-\! \hat{G}(x_i))^\top \! M \hat{G}(x)|| \!+\! ||\hat{G}(x_i)^\top M(\hat{G}(x) \!-\! \hat{G}(x_i))||  \nonumber \\
    & \leq 2\bar{G}||M|| M_G h 
\end{align}
Let $\epsilon_{pd} = 2\bar{G}||M|| M_G h $, if $\forall x_i \in \mathcal{B}\,, K(x_i,\theta^{*}) +\epsilon_{pd} I \prec  - \epsilon_i I$, then $K(x_i,\theta^{*}) \prec - \epsilon_i I$ holds for all $x$ in the entire state space $\mathcal{X}$. By Theorem~\ref{thm:estimated_system_stability}, stability holds for the true system, i.e., $G(x)^\top M G(x) - M  \prec 0$ for all $x \in \mathcal{X}$.
\end{proof}

After stating these key guarantees, we finally provide the finite-sample stabilization guarantee of \alg.

\begin{theorem}[Finite Sample Stabilization via \alg] \label{thm:sample_complexity} Suppose Assumptions \ref{asm:learning}-\ref{asm:regularity} hold and the batch $\mathcal{B}$ is informative enough that its fill distance $h$ satisfies $\bar{\epsilon} - \epsilon_{pd} > 0$, for $\epsilon_{pd} = 2\bar{G}||M|| M_G h $. Set $\epsilon_i = \bar{\epsilon} - \epsilon_{pd}$ in the constraint \eqref{eq:stability_LMI_Kra}. If \alg uses $ D = \mathcal{\tilde{O}} \left(\left( \frac{2\bar{G} \Vert M\Vert (1 + L_u)  + \Vert M\Vert  (1 + L_u)^2 }{\bar{\epsilon}-\epsilon_{pd}} \right)^3 \right)$ number of RFF in learning the system, after $D^2$ samples (time-steps), we have the trajectory of $x_{t+1} = f(x_t,g_\theta(x_t))$ is asymptotically stable around set $S_e = \{x:f(x,g(x)) = x\}$, i.e., the solution of \eqref{eq:policy_opt_stab} after $D^2$ samples from the system gives a stabilizing controller $g_{\theta}$ for the unknown nonlinear dynamical system.
\end{theorem}

\begin{proof}
Recall Assumption \ref{asm:stable} which states that the stability condition holds for the underlying system with $\bar{\epsilon}$ margin. Therefore, combining Theorem \ref{thm:estimated_system_stability} and Theorem \ref{thm:primal_dual_guarantee}, to guarantee the stabilization of the underlying system for the entire state-space, we require $\epsilon_i \leq \bar{\epsilon} - \epsilon_{pd} $, \textit{i.e.},
\begin{equation}
    \frac{\bar{\epsilon}-\epsilon_{pd}}{2\bar{G} \Vert M\Vert (1 + L_u)  + \Vert M\Vert  (1 + L_u)^2 } \geq \varepsilon_J,
\end{equation}
since $\varepsilon_J <1$. Note that this gives an upper bound on the error of Jacobian estimates to guarantee stabilization. From Proposition \ref{thm:RFF_jacobian_finite}, we also have that $\varepsilon_J = \mathcal{\tilde{O}}\left(\left(1/\sqrt{D}+\sqrt{D/\tau}\right)^{2/3}\right)$, since $\widehat{ \frac{\partial F(\phi)}{\partial x}}  - \frac{\partial F(\phi)}{\partial x}$ and $ \widehat{\frac{\partial F(\phi)}{\partial u}}  - \frac{\partial F(\phi)}{\partial u} $ are submatrices of $\hat{J}_F(\phi) - J_F(\phi)$. Furthermore, the optimal choice of $\tau$ and $D$ that minimizes this upped bound is $\tau = D^2$, which results that $\varepsilon_J = \mathcal{\tilde{O}}\left(D^{-1/3}\right)$ after $\tau = D^2$ samples. Therefore, for the choice of 
\begin{equation*}
   D = \mathcal{\tilde{O}} \left(\left( \frac{2\bar{G} \Vert M\Vert (1 + L_u)  + \Vert M\Vert  (1 + L_u)^2 }{\bar{\epsilon}-\epsilon_{pd}} \right)^3 \right)
\end{equation*}
number of RFF after $\tau = D^2$ time-steps \alg is guaranteed to stabilize the underlying system dynamics.
\end{proof}

This result shows that by setting the epoch length $\tau=D^2$, \alg guarantees the recovery of a stabilizing controller at the end of first epoch, i.e. $g_{\theta_1}$. Furthermore, the choice of $\epsilon_i$ in Theorem \ref{thm:sample_complexity} guarantees the recovery of stabilizing controllers for the subsequent epochs with the non-increasing behavior of the estimation errors.



\section{Conclusions and Future Work}
\label{sec:comclusion}
In this paper, we study the online control of unknown nonlinear dynamical systems. We propose a model-based RL framework, Krasovskii Constrainted RL (KCRL), that provides the first sample complexity result for stabilization of unknown nonlinear systems. KCRL iteratively learns the unknown model dynamics via RFF representation, and solves the policy optimization using the primal-dual approach, with a stability constraint generated by Krasovskii's construction of Lyapunov function. 
There are many interesting open questions that remain. 
For example, 
in practice the Lyapunov function design is a challenging task which requires heuristics and Krasovskii's method is only one way to construct Lyapunov functions. 
It is important to broaden the framework and consider how to incorporate other Lyapunov function construction techniques. In particular, Krasovskii's method is related to contraction analysis where the matrix $M$ is allowed to be time-dependent. This can be viewed as a generalization of the Krasovskii's method and has been widely used in robotic manipulation and locomotion tasks.

\bibliographystyle{unsrtnat} 
\bibliography{main}

\newpage
\appendix

\section{Proofs for Stability}\label{sec:stability_proofs}

In the Appendix, we show that the stability constraint of \alg, \eqref{eq:stability_LMI_Kra}, guarantees the stability of the underlying system using the true system dynamics, Theorem \ref{thm:true_system_stability}.

\begin{proposition}[Kowalewski's Mean Value Theorem \cite{kowalewski1895mittelwertsatz}]
\label{thm:MeanValueKow}
Let $x_1, \ldots, x_n$ be continuous functions in a variable $t\in [a,b]$. There exists real numbers $t_1, \ldots, t_n$ in $[a,b]$ and non-negative $\lambda_1, \ldots, \lambda_n$, with $\sum_{i}^n \lambda_i = b-a$, such that 
\begin{equation*}
    \int_a^b x_k(t) dt = \sum_{i=1}^n \lambda_i x_k(t_i)
\end{equation*}
for each $k=1,\ldots,n$.
\end{proposition}

\subsection{Proof of Theorem \ref{thm:true_system_stability}}

\begin{proof}
The key technique that underpins the proposed framework is to use Krasovskii's method for Lyapunov function construction to certifty stability. Note that Krasovskii's method considers continuous time dynamical systems, \textit{i.e.}, $\frac{d}{dt} x(t) = h(x(t))$. For this system, Krasovskii's method considers the candidate Lyapunov function of $V_h(x) = h^\top(x) P h(x)$, for some positive definite matrix $P \succ 0$. We first convert this approach to discrete-time nonlinear dynamical systems. To this end, notice that $h(x) = x - f(x)$ mimics the continuous case. In particular, let $x_e$ be the fixed point for $f(x)$, i.e., $f(x_e) = x_e$, which makes $h(x_e) = 0$. Following the Krasovskii's method, this gives the candidate Lyapunov function of 
\begin{equation}
    V_f(x) = (x-f(x))^\top M (x-f(x))
\end{equation}
for the underlying system $f(x)$. Note that we are considering the non-autonomous systems $F(\phi)$, where $\phi = [x^\top, u^\top]^\top$ given the controller $u = g_\theta(x)$. Let $F_\theta(x)$ denote the closed-loop system dynamics obtained via controller $g_\theta(x)$, i.e. $F_\theta(x) = F(\phi)$ with $\phi = [x^\top, g_\theta(x)^\top]^\top $. Therefore, we consider the following Lyapunov function:
\begin{equation}\label{eq:Kra_func}
    V(x) = (x-F_\theta(x))^\top M (x-F_\theta(x)).
\end{equation}
Firstly, note that $V(x)\geq 0$, and $V(x) = 0$ if and only if $x\in S_e$ as $M$ is positive definite. Then, we consider
\begin{align}\label{eq:lya_diff}
    V(x_{t+1}) - V(x_t) &= (x_{t+1}-F_\theta(x_{t+1}))^\top M (x_{t+1}-F_\theta(x_{t+1})) - (x_t - F_\theta(x_{t}))^\top M (x_t-F_\theta(x_{t})).
\end{align} 

First, consider the following for $h(x)= x - F_\theta(x)$. One can write $h(x_{t+1})$ in terms of $h(x_t)$ as follows, 
\begin{equation*}
    h(x_{t+1}) = h(x_t) + \int_{0}^1 \frac{\partial h}{\partial x} (x_t + t(x_{t+1}- x_t)) (x_{t+1} - x_t) dt.
\end{equation*}

From Proposition \ref{thm:MeanValueKow}, we have 
\begin{equation*}
    h(x_{t+1}) = h(x_t) + J_h (x_{t+1}-x_t),
\end{equation*}
where $J_h = \sum_{i=1}^n \lambda_i \frac{\partial h}{\partial x}(x_t + k_i(x_{t+1} - x_t))$ for $k_i \in [0,1]$, $\lambda_i \geq 0$ for all $i$ and $\sum_i^n \lambda_i = 1$. Plugging this in $V(x_{t+1})$, we get 
\begin{align*}
  V(x_{t+1}) &= h(x_t)^\top M h(x_t) + 2(x_{t+1}-x_t)^\top J_h^\top M h(x_t) + (x_{t+1}-x_t)^\top J_h^\top M J_h (x_{t+1}-x_t)
\end{align*}

Note that $x_{t+1} - x_t = F_\theta(x_{t}) -x_t = -h(x_t)$. Therefore, we get 
\begin{align*}
     V(x_{t+1}) &= h(x_t)^\top M h(x_t) - 2h(x_t)^\top J_h^\top M h(x_t) + h(x_t)^\top J_h^\top M J_h h(x_t) \\
     &= h(x_t)^\top (I-J_h)^\top M (I-J_h) h(x_t)
\end{align*}

From the definition of $J_h$ and $h(x)= x - F_\theta(x)$, we have $J_h = I - J_{F_\theta}$, where $J_{F_\theta} = \sum_{i=1}^n \lambda_i \frac{\partial F_\theta}{\partial x}(x_t + k_i(x_{t+1} - x_t))$, where $k_i$ and $\lambda_i$ follow from the definition of $J_h$. Thus we get 
\begin{equation}
    V(x_{t+1}) = (x_t - F_\theta(x_t))^\top J_{F_\theta}^\top M J_{F_\theta} (x_t-F_\theta(x_t)).
\end{equation}

Plugging this in \eqref{eq:lya_diff} gives 
\begin{equation}
    V(x_{t+1}) - V(x_t) = (x_t - F_\theta(x_t))^\top (J_{F_\theta}^\top M J_{F_\theta} - M) (x_t-F_\theta(x_t)).
\end{equation}

For any $x\in \mathbb{R}^n$, we have 
\begin{align*}
    x^\top J_{F_\theta}^\top M J_{F_\theta} x &= \left \| M^{1/2} J_{F_\theta} x \right \|^2 \\
    &= \left \| \sum_{i=1}^n \lambda_i M^{1/2} \frac{\partial F_\theta}{\partial x} \bigg(x_t + k_i(x_{t+1} - x_t)\bigg) x \right \|^2 \\
    &\leq \sum_{i=1}^n \lambda_i \left \| M^{1/2} \frac{\partial F_\theta}{\partial x} \bigg(x_t + k_i(x_{t+1} - x_t)\bigg) x \right \|^2 \\
    &= \sum_{i=1}^n \lambda_i x^\top \frac{\partial F_\theta}{\partial x} \bigg(x_t + k_i(x_{t+1} - x_t)\bigg)^\top M \frac{\partial F_\theta}{\partial x} \bigg(x_t + k_i(x_{t+1} - x_t)\bigg) x,
\end{align*}
where the inequality is due to Jensen's inequality. Due to the constraint \eqref{eq:stability_LMI_Kra}, we have that $\frac{\partial F_\theta}{\partial x}(x)^\top M \frac{\partial F_\theta}{\partial x}(x)  \prec M -\epsilon I $ for all $x\in\mathbb{R}^n$. Thus, we have that  
\begin{align}
    J_{F_\theta}^\top M J_{F_\theta} - M \preceq \sum_{i}  \lambda_i \left[ \frac{\partial F_\theta}{\partial x} \bigg(x_t + k_i(x_{t+1} - x_t)\bigg)^\top M \frac{\partial F_\theta}{\partial x} \bigg(x_t + k_i(x_{t+1} - x_t)\bigg) - M \right]  \preceq -\epsilon I. \label{eq:stability_guarantee} 
\end{align}

This shows that the Lyapunov function is decreasing along the system trajectory, i.e. $ V(x_{t+1}) - V(x_t) = (x_t - F_\theta(x_t))^\top (J_{F_\theta}^\top M J_{F_\theta} - M) (x_t-F_\theta(x_t)) < 0$. 

Lastly, if a trajectory $x_t$ is such that $V(x_{t+1}) - V(x_t) = 0, \forall t\geq 0$, then we must have $F_\theta(x_t) = x_t$ for all $t$, i.e. $x_t\in S_e$ for all $t$. Therefore, by LaSalle's Invariance Principle, we must have $S_e$ is asymptotically stable. 

\end{proof}
\end{document}